\documentclass[11pt]{article}

\def\colorful{0}

\usepackage[letterpaper,margin=1in]{geometry}

\usepackage[utf8]{inputenc} \usepackage[T1]{fontenc}        \usepackage{url}            \usepackage{booktabs}       \usepackage{amsfonts}       \usepackage{nicefrac}       \usepackage{microtype}      \usepackage{xcolor}

\usepackage{amsthm,amsmath,amssymb,amsfonts,amssymb,mathtools}
\usepackage{xcolor}
\usepackage{empheq}
\usepackage{dsfont}
\usepackage{mathrsfs}
\usepackage{graphicx}
\usepackage{enumitem}
\usepackage{framed}
\usepackage{algorithm2e}
\usepackage[noend]{algpseudocode}
\usepackage[colorlinks,citecolor=blue,linkcolor=magenta,bookmarks=true]{hyperref}
\usepackage[nameinlink]{cleveref}
\crefname{ineq}{Inequality}{Inequality}
\creflabelformat{ineq}{#2{\upshape(#1)}#3}
\crefname{sub}{Subsection}{Subsection}
\creflabelformat{Subsection}{#2{\upshape(#1)}#3}
\crefname{sdp}{SDP}{SDP}
\creflabelformat{sdp}{#2{\upshape(#1)}#3}
\crefname{lp}{LP}{LP}
\creflabelformat{lp}{#2{\upshape(#1)}#3}
\usepackage{algorithmicx}
\usepackage[noend]{algpseudocode}
\makeatletter
\newenvironment{Ualgorithm}[1][htpb]{\def\@algocf@post@ruled{\kern\interspacealgoruled\hrule  height\algoheightrule\kern3pt\relax}\def\@algocf@capt@ruled{under}\setlength\algotitleheightrule{0pt}\SetAlgoCaptionLayout{centerline}\begin{algorithm}[#1]}
	{\end{algorithm}}
\makeatother

\newtheorem{theorem}{Theorem}[section]

\newtheorem{lemma}[theorem]{Lemma}
\newtheorem{informal theorem}[theorem]{Theorem (informal statement)}

\newtheorem{proposition}[theorem]{Proposition}

\newtheorem{fact}[theorem]{Fact}
\newtheorem{claim}[theorem]{Claim}

\newtheorem{remark}[theorem]{Remark}

\theoremstyle{definition}
\newtheorem{definition}[theorem]{Definition}
\newcommand{\eqdef}{\stackrel{{\mathrm {\footnotesize def}}}{=}}

\renewcommand\vec[1]{\mathbf{#1}}
\DeclareMathOperator*{\pr}{\mathbf{Pr}}
\DeclareMathOperator*{\E}{\mathbf{E}}
\newcommand{\proj}{\mathrm{proj}}

\DeclareMathOperator*{\argmin}{argmin}

\newcommand{\bx}{\mathbf{x}}
\newcommand{\by}{\mathbf{y}}

\newcommand{\bw}{\mathbf{w}}

\newcommand{\err}{\mathrm{err}}

\newcommand{\R}{\mathbb{R}}

\newcommand{\Z}{\mathbb{Z}}

\newcommand{\eps}{\epsilon}

\newcommand{\poly}{\mathrm{poly}}

\newcommand{\sgn}{\mathrm{sign}}
\newcommand{\sign}{\mathrm{sign}}

\newcommand{\opt}{\mathrm{opt}}
\newcommand{\D}{{D}}

\newcommand{\Ind}{\mathds{1}}
\newcommand{\1}{\Ind}

\newcommand{\littlesum}{\mathop{\textstyle \sum}}

\newcommand{\wh}{\widehat}

\newcommand{\wstar}{{\vec w}^{\ast}}

\newcommand{\x}{\vec x}

\newcommand{\w}{\vec w}

\newcommand{\Exx}{\E_{\x\sim \D_\x}}
\newcommand{\Ey}{\E_{(\x,y)\sim \D}}

\def\colorful{0}

\ifnum\colorful=1
\newcommand{\new}[1]{{\color{red} #1}}

\else
\newcommand{\new}[1]{{#1}}

\fi

\title{A Near-optimal Algorithm for Learning Margin Halfspaces with Massart Noise\thanks{A conference version of this work appears in the proceedings of the Thirty-Eighth Annual Conference on Neural Information Processing Systems (NeurIPS 2024).}}

\author{Ilias Diakonikolas\thanks{Supported in part by NSF Medium Award CCF-2107079 and an H.I. Romnes Faculty Fellowship.} \\
  UW-Madison\\
  \texttt{ilias@cs.wisc.edu} \\
  \and
    Nikos Zarifis\thanks{Supported in part by NSF Medium Award CCF-2107079.} \\
  UW-Madison\\
  \texttt{zarifis@wisc.edu} \\
}

\begin{document}

\maketitle

\begin{abstract}
  We study the problem of PAC learning $\gamma$-margin halfspaces in the presence of Massart noise. 
Without computational considerations, the sample complexity of this learning problem is known to be 
$\widetilde{\Theta}(1/(\gamma^2 \eps))$. 
Prior computationally efficient algorithms for the problem incur sample complexity 
$\tilde{O}(1/(\gamma^4 \eps^3))$ and achieve 0-1 error of $\eta+\eps$, 
where $\eta<1/2$ is the upper bound on the noise rate.
Recent work gave evidence of an information-computation tradeoff, 
suggesting that a quadratic dependence on $1/\eps$ is required 
for computationally efficient algorithms. 
Our main result is a computationally efficient learner with sample complexity 
$\widetilde{\Theta}(1/(\gamma^2 \eps^2))$, nearly matching this lower bound. 
In addition, our algorithm is simple and practical, 
relying on online SGD on a carefully selected sequence of convex losses.  \end{abstract}
\setcounter{page}{0}
\thispagestyle{empty}
\newpage

\section{Introduction}

This work studies the algorithmic task 
of learning margin halfspaces in the presence of Massart noise (aka bounded label noise)~\cite{Massart2006}
with a focus on fine-grained complexity analysis. A halfspace or Linear Threshold Function (LTF) is any 
Boolean-valued function $h: \R^d \to \{ \pm 1\}$ of the form 
$h(\bx) = \sgn \left( \bw \cdot \bx -\theta \right)$, where $\bw \in \R^d$ is the weight vector and $\theta \in \R$ is the threshold. 
The function $\sign: \R \to \{ \pm 1\}$ is defined as $\sgn(t)=1$ if $t \geq 0$ 
and $\sgn(t)=-1$ otherwise.
The problem of learning halfspaces with a margin --- i.e., under the assumption 
that no example lies too close to the separating hyperplane --- is one of the earliest algorithmic problems
studied in machine learning, going back to the Perceptron algorithm~\cite{Rosenblatt:58}.

In the realizable PAC model~\cite{val84} (i.e., with clean labels), 
the sample complexity of learning $\gamma$-margin halfspaces on the unit ball in $\R^d$
is $\Theta(1/(\gamma^2\eps))$, where $\eps>0$ is the desired 0-1 error; 
see, e.g.,~\cite{SB14}\footnote{As is standard, we are 
assuming  that $d = \Omega (1/\gamma^2$); otherwise, a sample complexity 
bound of $\widetilde{O}(d/\eps)$ follows from standard VC-dimension arguments.}. 
Moreover, the Perceptron algorithm is a computationally efficient 
learner achieving this sample complexity. 
That is, without label noise, there is a sample-optimal and computationally 
efficient learner for margin halfspaces.

In this paper, we study the same problem in the Massart noise model that we now define.

\begin{definition}[PAC Learning with Massart Noise]\label{assmpt:Massart}
Let $\D$ be a distribution over $\mathcal X\times\{\pm 1\}$, and let 
$\mathcal C$ be a class of Boolean-valued functions over 
$\mathcal X$. We say that $\D$
satisfies the $\eta$-Massart noise condition with respect to $\mathcal{C}$, for some $\eta<1/2$, if there 
exists a concept $f \in \mathcal{C}$ and an unknown noise
function $\eta(\x):\mathcal X \mapsto[0,\eta]$ such that 
for $(\x,y) \sim \D$, the label $y$ satisfies: 
with probability $1-\eta(\x)$,  $y=f(\x)$; and $y=-f(\x)$ otherwise. 
Given i.i.d.\ samples from $\D$, the goal of the learner is to output a hypothesis $h: \mathcal X \to \{\pm 1\}$ such that
with high probability the 0-1 error
$\err_{\D}(h) \eqdef \pr_{(\bx, y) \sim \D}[h(\bx) \neq y]$ is small.
\end{definition}

The concept class of halfspaces with a margin is defined as follows.

\begin{definition}[$\gamma$-Margin Halfspaces]\label{assmpt:margin+Massart}
Let $\D$ be a distribution over $\mathbb{S}^{d-1}\times\{\pm 1\}$, 
where $\mathbb{S}^{d-1}$ is the unit sphere in $\R^d$. 
Let $\wstar\in\mathbb{S}^{d-1}$ and $\gamma\in(0, 1)$. 
We say that the distribution $\D$ satisfies the $\gamma$-margin condition 
with respect the halfspace $\sign(\wstar\cdot\x)$\footnote{We will henceforth assume that the threshold is $\theta = 0$, 
which is well-known to be no loss of generality.},  
if (i) for $(\x,y) \sim D$, we have that $y = \sign(\wstar\cdot\x)$, and 
(ii) $\pr_{(\x,y) \sim \D} \left[ |\wstar \cdot \x| < \gamma \right] = 0$. 
The parameter $\gamma$ is called the margin of the halfspace $\sign(\wstar\cdot\x)$. 
\end{definition}

Information-theoretically, the best possible 0-1 error attainable for learning a concept class
with Massart noise is $\opt: = \E_{\x \sim \D_{\x}}[\eta(\x)]$. Since $\eta(\x)$ is uniformly bounded above by $\eta$, it follows that $\opt \leq \eta$; also note that it may well be the case that $\opt \ll \eta$. 
Focusing on the class of $\gamma$-margin halfspaces, 
it follows from~\cite{Massart2006} that there exists a (computationally inefficient) estimator achieving
error $\opt+\eps$ with sample complexity $\widetilde{O}(1/((1-2\eta)\gamma^2 \eps))$; and moreover that 
this sample upper bound is nearly best possible (within a logarithmic factor) for any estimator. 
(That is, the sample complexity of the Massart learning problem is essentially the 
same as in the realizable case, as long as $\eta$ is bounded from $1/2$.) 

Taking computational considerations into account, the feasibility landscape of the problem changes.
Prior work~\cite{DK20-SQ-Massart,NasserT22,DKPR22}
has provided strong evidence that 
achieving error better than $\eta+\eps$ is not possible in polynomial time. Consequently, algorithmic research 
has been focusing on achieving the qualitatively weaker error guarantee of $\eta+\eps$. We note that efficiently obtaining any non-trivial  guarantee had remained open since the 80s; see \Cref{app:related} for a discussion. The first algorithmic 
progress for this problem is due to~\cite{DGT19}, who gave a 
polynomial-time algorithm achieving error 
of $\eta+\eps$ with sample complexity $\poly(1/\gamma, 1/\eps)$. 
Subsequent work~\cite{CKMY20} gave an efficient 
algorithm with improved sample complexity of $\tilde{O}(1/(\gamma^4 \eps^3))$. 
Prior to the current work, 
this remained the best known sample upper bound for efficient algorithms. 

In summary, known computationally efficient algorithms for learning margin 
halfspaces with Massart noise require significantly more samples---namely, 
$\tilde{\Omega}(1/(\gamma^4 \eps^3))$---than the information-theoretic 
minimum of $\widetilde{\Theta}_{\eta}(1/(\gamma^2 \eps))$. 
It is thus natural to ask whether a polynomial-time algorithm with 
optimal (or near-optimal, i.e., within logarithmic factors) sample complexity exists. Recall that the answer to this 
question is affirmative in the realizable setting, where the Perceptron algorithm is optimal. Perhaps 
surprisingly, recent work~\cite{DDKWZ23} (see also~\cite{DDKWZ23b}) gave evidence for the existence of inherent 
{\em information-computation tradeoffs} in the Massart noise model---in fact, even in the simpler model of Random 
Classification Noise (RCN)~\cite{AL88}\footnote{The RCN model is the special case of Massart noise, where 
$\eta(\x) = \eta$ for all points $\x$ in the domain.}. 
Specifically, they showed that any efficient 
Statistical Query (SQ) algorithm or low-degree polynomial tasks requires 
$\Omega(1/\eps^2)$ samples---a near quadratic blow-up 
compared to the $\tilde{O}(1/\eps)$ information-theoretic upper bound.
This discussion serves as the motivation for the following question:
\begin{center}
{\em What is the optimal {\em computational sample complexity} of the problem of \\ 
learning $\gamma$-margin halfspaces with Massart noise?}
\end{center}
By the term ``computational sample complexity'' above, we mean the sample complexity of
polynomial-time algorithms for the problem. 
Given the fundamental nature of this learning problem, we believe that a fine-grained 
sample complexity versus computational complexity analysis is interesting on its own merits.
{\em In this work, we develop a computationally efficient algorithm with 
sample complexity of $\tilde{O}(1/(\gamma^2 \eps^2))$.} Given the aforementioned 
information-computation tradeoffs, there is evidence that this upper bound is 
close to best possible. As a bonus, our algorithm is also simple and 
practical, relying on online SGD. (In fact, our algorithm runs in sample 
linear time, excluding a final testing step that slightly increases the runtime.)

\subsection{Our Result and Techniques} \label{ssec:res}

Our main result is the following:

\begin{theorem}[Main Result, Informal]\label{thm:main-inf} 
Let $D$ be a distribution on $\mathbb{S}^{d-1} \times \{\pm 1\}$ that satisfies 
the $\eta$-Massart noise condition with respect to an unknown $\gamma$-margin 
halfspace $f(\x) = \sign(\wstar\cdot\x)$. There is algorithm that draws 
$n = \tilde{O} (1/(\eps^2 \gamma^2))$ samples from $D$, runs in time $\tilde{O}(dn/\eps)$,
and with probability at least $9/10$ returns a vector $\hat{\vec w}$ such that 
$\mathrm{err}_{\D}(\hat{\vec w})\leq \eta+\eps$. 
\end{theorem}

The sample upper bound of \Cref{thm:main-inf} nearly matches the computational 
sample complexity of the problem (for SQ algorithms and low-degree 
polynomial tests), which was shown to be $\Omega(1/(\eps^2\gamma)+ 
1/(\eps\gamma^2))$~\cite{Massart2006, DDKWZ23, DDKWZ23b}.
That is, \Cref{thm:main-inf} comes close to resolving 
the fine-grained complexity of this basic task. 
Moreover, it matches known algorithmic guarantees for the easier case of Random Classification Noise~\cite{DDKWZ23, KITBMV23}.

\paragraph{Independent Work }
Independent work \cite{Kontonis2024} obtained a learning algorithm for 
$\gamma$-margin halfspaces with essentially the same sample and computational complexity as ours.

\paragraph{Brief Overview of Techniques} 
Here we provide a brief summary of our approach in tandem with a comparison to prior work.
The algorithm of~\cite{DGT19} adaptively partitions the space into polyhedral regions and 
uses a different linear classifier in each region, 
each achieving error $\eta+\eps$ within the corresponding region. 
Their approach leverages the LeakyReLU loss
(see \eqref{eq:old-loss}) as a convex proxy to the 0-1 loss. 
At a high-level, their approach reweights the samples in order 
to accurately classify a non-trivial fraction of points. 
\cite{CKMY20} uses the LeakyReLU loss to efficiently identify 
a region where the value of the loss conditioned on this region is 
sub-optimal; they then use this procedure as a separation oracle 
along with online convex optimization (see also~\cite{DKTZ20b, DKKTZ21b}) to output 
a linear classifier with 0-1 error at most $\eta+\eps$. 
Both of these approaches inherently require $\Omega(1/\eps^3)$ samples for the following reason: 
they both need to condition on a region where the probability mass of the distribution can be as small as $\Theta(\eps)$. Thus, even estimating the error 
of the loss would require at least $\Omega(1/\eps^2)$ conditional samples.  
Beyond the dependence on $1/\eps$, the sample complexity achieved in these prior works 
is also suboptimal in the margin parameter $\gamma$; namely, $\Omega(1/\gamma^4)$. 
This dependence follows from the facts that both of these works require 
estimating the loss in each iteration within error of at most 
$\gamma\eps$, and that their algorithmic approaches require $\Omega(1/\gamma^2)$ iterations. 

To circumvent these issues, novel ideas are required. 
At a high-level, we design a uniform approach to decrease the ``global'' 
error, as opposed to the local error (as was done in prior work). 
Specifically, we construct a different sequence of convex loss functions, 
each of which attempts to accurately simulate the 0-1 objective. 
We note that a similar sequence of loss functions was used 
in the recent work~\cite{DKTZonline} in a related, but significantly different, 
adversarial online setting. Interestingly, a similar reweighting scheme 
was used in~\cite{CKMY20} for learning general Massart halfspaces.  
Beyond this similarity, these works have no implications for the sample
complexity of our problem. (See \Cref{ssec:related} for a detailed comparison.) 
Via this approach, we obtain an iterative algorithm which uses 
only $O_{\gamma}(1/\eps^2)$ samples in order to estimate the loss in each iterative step. 

In more detail, note that the 0-1 loss can be written in the form
$-\E[y \frac{\vec w \cdot \x}{|\vec w\cdot \x|} ]$. 
We convexify this objective by considering, in each step, the loss  
$\ell(\vec w,\vec u)=-\E[y \frac{\vec w \cdot \x}{|\vec u\cdot \x|} ]$, 
where $\vec u$ is independent of $\vec w$; 
this loss is convex with respect to $\vec w$. 
Observe that $\ell(\vec w,\vec w)$ is proportional to the zero-one loss 
of $\vec w$. Unfortunately, it is possible that no optimal 
vector $\vec w^*$ (under 0-1 loss) 
minimizes $\ell(\vec w^*,\vec w)$. 
For this reason, we consider the objective 
$\ell_{\eta}(\vec w,\vec u)=
\E[(\1\{y\neq \sign(\vec w\cdot\x)\}-\eta-\eps) |\vec w \cdot \x|/|\vec u\cdot \x|]$. 
This new objective satisfies the following: 
$\ell_\eta(\w^*,\vec u)<-\eps\gamma$ for any vector $\vec u$ 
and any $\wstar$ that minimizes the 
0-1 objective; and $\ell_\eta(\vec w,\vec w)\geq \eps$ 
as long as $\vec w$ incurs 0-1 error at least $\eta+\eps$. 
By the convexity of $\ell_\eta(\vec w,\vec u)$, 
this allows us to construct a separation oracle.  Namely, we draw enough samples so that 
$\wh{\ell}_\eta(\vec w,\vec w)-\wh{\ell}_\eta(\wstar,\vec w)\geq \eps/2$, where 
$\wh{\ell}$ is the emprical version of the loss.  Due to the nature of these objectives, 
$O_{\gamma}(1/\eps^2)$ samples per iteration suffice for this purpose. 
This in turn implies that the cutting planes method efficiently finds 
a near-optimal weight vector after $O(\log(1/\eps)/\gamma^2)$ iterations. 
Overall, this approach leads to an efficient algorithm with 
sample complexity $\tilde{O}_{\gamma}(1/\eps^2)$. To get the desired
sample complexity of $\tilde{O}(1/(\eps^2\gamma^2))$, more ideas are needed.

In the previous paragraph, we hid an obstacle 
that makes the above approach fail. Specifically, 
it may be possible that, for many points $\x$, 
the value of $|\vec u\cdot\x|$ is arbitrarily small. 
To fix this issue, we consider a clipped reweighting as follows: 
$\ell'_{\eta}(\vec w,\vec u)=\E[(\1\{y\neq \sign(\vec w\cdot\x)\}-\eta-\eps) \frac{|\vec w \cdot \x|}{\max(|\vec u\cdot \x|,\gamma)} ]$. This clipping step is not a problem for us,
because the target halfspace $\sgn(\vec w ^{\ast} \cdot \x)$ was assumed to have margin $\gamma$. 
This guarantees that the difference between the expected (over $y$) 
pointwise losses at $(\vec w,\vec w)$ and $(\wstar,\vec w)$ 
is at least $\eps$ on the points $\x$ where $|\vec u\cdot\x|\leq \gamma$. 
Indeed, when this is the case, then $|\wstar\cdot\x|/|\vec u\cdot\x|\geq 1$. 
Overall, this suffices to guarantee that 
${\ell}'_\eta(\vec w,\vec w)-{\ell}'_\eta(\wstar,\vec w)\geq \eps$.

\subsection{Notation}\label{ssec:prelims}

For $n \in \Z_+$, let $[n] \eqdef \{1, \ldots, n\}$.
We use small
boldface characters for vectors.
For $\bx \in \R^d$ and $i \in [d]$, $\bx_i$
denotes the $i$-th coordinate of $\bx$, and $\|\bx\|_2 \eqdef (\littlesum_{i=1}^d \bx_i^2)^{1/2}$ denotes the $\ell_2$-norm of $\bx$.
We will use $\bx \cdot \by $ for the inner product of $\bx, \by \in \R^d$. For a subset $S\subseteq \R^d$, we define the $\proj_{S}$ operator that maps a point $\x\in \R^d$ to the closest point in the set $S$. 
For $a,b\in \R$, we denote $W(a,b)\eqdef 1/\max(a,b)$.
We will use $\1_A$ to denote the characteristic function of the set $A$,
i.e., $\1\{\x \in A\}= 1$ if $\x\in A$, and $\1\{\x \in A\}= 0$ if $\x\notin A$.
For $A,B\in \R$, we write $A\gtrsim B$ (resp. $A\lesssim B$) to denote that there exists a universal constant $C>0$, such that $A\geq C B$ (resp. $A\leq C B$).

We use $\E_{x\sim \D}[x]$ for the expectation of the random variable $x$ with respect to the distribution $\D$ and
$\pr[\mathcal{E}]$ for the probability of event $\mathcal{E}$. 
For simplicity, we may
omit the distribution when it is clear from the context. 
For $(\x,y) \sim \D$, we use $\D_\x$ for the marginal distribution of $\x$ 
and $\D_y(\x)$ for the distribution of $y$ conditioned on $\x$. 
We use $\widehat D_N$ to denote the empirical distribution obtained by drawing 
$N$ i.i.d.\ samples from $D$. 
We use $\mathrm{err}_D(\vec w)$ to denote 
the 0-1 error of the halfspace defined by the weight vector $\vec w$ 
with respect to the distribution $D$, 
i.e., $\mathrm{err}_D(\vec w)\eqdef \pr_{(\x,y)\sim D}[\sign(\vec w\cdot\x)\neq y]$. 
We will use $\mathrm{err}(\w,\x)$ for the 0-1 error of $\sign(\vec w\cdot\x)$ 
conditioned on $\x$, 
i.e., $\mathrm{err}(\w,\x):=\pr_{y\sim D_y(\x)}[\sign(\vec w\cdot\x)\neq y]$. 
Note that $\mathrm{err}_D(\vec w)=\Exx[\mathrm{err}(\w,\x)]$. If $D$ satisfies
the $\eta$-Massart noise condition with respect to the halfspace 
$\sign(\vec w\cdot\x)$, 
then $\mathrm{err}(\w,\x)= \eta(\x) \1\{\sign(\w\cdot\x)=\sign(\wstar\cdot\x)\}
+ ( 1-\eta(\x)) \1\{\sign(\w\cdot\x)\neq\sign(\wstar\cdot\x)\} \;.$

\section{Our Algorithm and its Analysis: Proof of~\Cref{thm:main-inf}} \label{sec:main}
In this section, we prove our main result. 
\Cref{alg:massart} efficiently learns the class of margin halfspaces 
on the unit ball, in the presence of Massart noise, 
with sample complexity nearly matching the information-computation limit.
Additionally, its runtime is linear in the sample size, 
excluding a final testing step to select the best hypothesis.

At a high-level,  
our algorithm leverages a carefully selected convex loss 
(or, more precisely, a sequence of convex losses) --- 
serving as a proxy to the 0-1 error. 
A common loss function, introduced in this context by~\cite{DGT19} 
and leveraged in~\cite{DGT19, CKMY20}, 
is the LeakyReLU function. This is the univariate function  $\mathrm{LeakyReLU}_\lambda(t)=(1-\lambda)\1\{t\geq 0\}t+\lambda\1\{t<0\}t$, 
where $\lambda\in(0,1)$ is the leakage parameter (that needs to be selected carefully). 
Roughly speaking, the convex function 
$\ell_{\lambda}(\vec w,\x,y)=\mathrm{LeakyReLU}_\lambda(-y (\vec w\cdot\x) )$ can be viewed as a reasonable proxy to the 0-1 loss of the halfspace $\sign(\vec w\cdot\x)$
on the point $(\x,y)$. To see this, note that (see, e.g., \Cref{clm:dgt-claim})
\begin{equation} 
    \ell_\lambda(\vec w,\x,y)=(\1\{\sign(\vec w\cdot\x)\neq y\}-\lambda)|\vec w\cdot \x|\;. \label{eq:old-loss}
\end{equation}
Observe that a point $\x$ that is classified correctly by the halfspace 
$\sign(\vec w \cdot \x)$ 
will satisfy
\[\big(\mathbf{E}_{y\sim D_y(\x)}[\1\{\sign(\vec w\cdot\x)\neq y\}]-\lambda\big)|\vec w\cdot \x|=(\eta(\x)-\lambda)|\vec w\cdot \x| \]
which is non-positive for $\lambda\geq \eta(\x)$. 
Since the only guarantee we have is that $\eta(\x) \leq \eta$, this suggests that we need to select $\lambda \geq \eta$. It turns out that $\lambda: = \eta$ is the optimal choice. 
We fix the choice of $\lambda: = \eta$ throughout. 
On the other hand, if (the halfspace defined by) $\vec w$ 
misclassifies the point $\x$, this term becomes non-negative.

The factor $|\vec w\cdot\x|$ in \Cref{eq:old-loss}  
reweights the 0-1 error 
so that points $\x$ for which $|\vec w\cdot\x|$ is sufficiently large
(i.e., close to $1$) have to be classified correctly 
by a minimizer of $\E_{(\x,y) \sim D}[\ell_\lambda(\vec w,\x,y)]$. 
On the other hand, points closer to the separating hyperplane defined by $\vec w$, 
or points where $\eta(\x)$ is close to $\lambda = \eta$, are not guaranteed 
to be classified correctly by the minimizer of this loss. 
We leverage this insight to construct a sequence of loss functions 
that reweight the points so that, to minimize the regret, 
we need to classify a large fraction of points; this leads to the desired
error of $\eta+\eps$ with near-optimal sample complexity.

We now provide some intuition justifying our choice of surrogate loss functions.
Observe that if we instead could minimize the function
\begin{equation}
  \Ey[\ell_\lambda(\vec w,\x,y)/|\vec w\cdot \x|]=\Ey[(\1\{\sign(\vec w\cdot\x)\neq y\}-\lambda)] \;, \label{eq:loss-and-01}
\end{equation}
with respect to $\vec w$, we would obtain a halfspace with 
minimum 0-1 error; 
unfortunately, this reweighted loss is just a shift of the 0-1 loss, hence 
non-convex. To fix this issue, instead of reweighting by $1/|\vec w\cdot\x|$, we will reweight by $W(\vec v\cdot\x,\gamma)\eqdef 1/\max(|\vec v\cdot\x|,\gamma)$, 
where $\gamma$ is the margin parameter and $\vec v$ is an appropriately chosen vector that is independent of $\vec w$. The new loss is defined as follows:  
\begin{equation}
    \mathcal L_{\lambda,\vec v}(\vec w)\eqdef \Ey[\ell_{\lambda}(\vec w,\x,y)W(\vec v\cdot\x,\gamma/2)]\;,\label{eq:loss}
\end{equation}
where for technical reasons we use $\gamma/2$ instead of $\gamma$ in the maximum.

Since the parameter $\vec v$ is independent of $\vec w$, 
the loss $\mathcal L_{\lambda,\vec v}(\vec w)$ remains convex in $\vec w$. 
At the same time, by carefully choosing $\vec v$, we can accurately simulate
the non-convex 0-1 loss. Note that our reweighting term is a maximum over two terms.  
The reason for this choice is that, for some points $\x$, the quantity $|\vec v\cdot\x|$ 
can be arbitrarily small; taking the maximum avoids the loss becoming very large. In 
particular, the loss $\mathcal L_{\lambda,\vec v}(\vec w)$  
will be guaranteed to remain in a bounded length interval. 

Our algorithm proceeds in a sequence of iterations. 
In the $(t+1)$-st iteration, it sets $\vec v$ to be $\vec w^t$, 
where $\vec w^t$ is the weight vector of step $t$. This choice attempts to
simulate the 0-1 error at $\vec w^t$, 
as is suggested by \Cref{eq:loss-and-01}. 
Assume for simplicity that our current hypothesis is the halfspace defined by $\vec w$ 
and is such that $\Exx[\1\{|\vec w\cdot\x|\leq \gamma/2\}]=0$. 
Note this implies that $W(\vec w\cdot\x,\gamma/2)=1/|\vec w\cdot \x|$. 
By combining \Cref{eq:loss,eq:loss-and-01}, we get that 
$\mathcal L_{\lambda,\vec w}(\vec w)=\err_{\D}(\vec w)-\lambda$; 
note that as long as $\err_{\D}(\vec w)\geq \lambda+\eps$, 
we have that $\mathcal L_{\lambda,\vec w}(\vec w)\geq \eps$. 
On the other hand, the optimal halfspace $\wstar$ achieves a non-positive loss; 
from \Cref{eq:old-loss,eq:loss-and-01}, we have that
\begin{align*}
    \mathcal L_{\lambda,\vec w}(\wstar)&=
\Ey[(\1\{\sign(\wstar\cdot\x)\neq y\}-\lambda)|\wstar\cdot\x|W(\vec w\cdot\x,\gamma/2)]
\\    &=\Exx[(\eta(\x)-\lambda)|\wstar\cdot\x|W(\vec w\cdot\x,\gamma/2)]\leq 0\;,
\end{align*} 
where the inequality follows from the fact that $\eta(\x)\leq \eta$. 
Recalling that $\mathcal L_{\lambda,\vec v}(\vec w)$ is convex, 
if we run an Online Convex Optimization (OCO) algorithm, 
after $T$ steps we are guaranteed to find a vector $\vec w$ such that 
$\mathcal L_{\lambda,\vec w}(\vec w)-\mathcal L_{\lambda,\vec w}(\wstar)\leq O(1/\sqrt{T})$. 
For $T=O(1/\eps^2)$, this gives that $\mathcal L_{\lambda,\vec w}(\vec w)<\eps/2$;  
and therefore we would have $\err_D(\vec w)<\lambda+\eps$. We provide an approach using this idea and the cutting planes algorithm in \Cref{ssec:online} that achieves sample complexity $\widetilde O(1/(\eps^2\gamma^4))$.

Our algorithm and its analysis work only with the gradient of 
$\mathcal L_{\lambda,\vec v}(\vec w)$. The key novelty is the analysis of the sample complexity. 
The gradient of $\ell_\lambda(\w,\x,y)W(\vec v\cdot\x,\gamma)$ 
with respect to $\vec w$
has the following explicit form:
\[ \vec g_{\lambda,\gamma}(\vec w,\vec v,\x,y)\eqdef 
((1-2\lambda)\sign(\vec w\cdot \x)-y)W(\vec v\cdot\x,\gamma)\x \,
=\frac{\left((1-2\lambda)\sign(\vec w\cdot \x)-y\right)}{\max(|\vec v\cdot \x|,\gamma)}\x \,.\] 
Furthermore, we denote by $\vec G_{\D}(\vec w,\vec v,\eta,\gamma)=\Ey[\vec g_{\eta,\gamma}(\vec w,\vec v,\x,y)]$.

Before describing our algorithm and proving \Cref{thm:main-detailed}, 
we simplify our notation. 
We will omit the parameters $\eta,\gamma$ 
from the function input (as they are fixed throughout).
Therefore, we use 
$\vec G_{\widehat D_N^t}(\vec w,\vec v)\equiv \vec G_{\widehat D_N^t}(\vec w,\vec v,\eta,\gamma)$ 
and 
$\vec g(\vec w,\vec v,\x,y)\equiv \vec g_{\eta,\gamma/2}(\vec w,\vec v,\x,y)$.

Our algorithm is described in pseudocode below.

\begin{Ualgorithm}[h]
	\centering
	\fbox{\parbox{5.7in}{
			{\bf Input:} Sample access to a distribution $D$ supported in $\mathbb{S}^{d-1}\times \{\pm 1\}$ corrupted with $\eta$-Massart noise with respect to a halfspace $\sign(\wstar\cdot\x)$ that satisfies the $\gamma$-margin condition; parameters $\eps,\delta\in(0,1)$, and $N,T\in \Z_+$.\\
			{\bf Output:} Weight vector $\hat{\vec w}$ such that $\mathrm{err}_{\D}(\hat{\vec w})\leq \eta+\eps$ with probability at least $1-\delta$.
			\begin{enumerate}
   \item Let $c>0$ be a sufficiently small universal constant. 
   \item $t\gets 0$, $\vec w^0 \gets \vec e_1 = (1, 0, \ldots, 0)$, and $T=(1/c)\log(1/\delta)/(\eps^2\gamma^2)$.
				\item While  $t\leq T$ do
				\begin{enumerate}
    				\item Draw $(\vec x^{(t)},y^{(t)})$ sample from $D$.
\item Set $\lambda_t \gets c\gamma^2\eps$.
					\item Update $\vec w^t$ as follows:\label{alg:estimate3} \qquad  \qquad\quad$\triangleright$ Update and project in the unit ball\begin{align*}
					    \vec v^{t+1}&\gets\vec w^{t}-\lambda_t \vec g(\vec w^t,\vec w^t,\x^{(t)},y^{(t)})
        \quad \quad\vec w^{t+1}\gets\frac{\vec v^{t+1}}{\max(\|\vec v^{t+1}\|_2,1)}
					\end{align*}
					 \label{alg:update}
					\item $t\gets t+1$. 
		\end{enumerate}
			\item Draw $N$ samples from $D$ and construct the empirical distribution $\widehat D_N$.
   \item Return $\widehat{\vec w}=\argmin_{t\in[T+1]}\err_{\widehat D_N}(\w^t)$. \label{alg:tournament}
	\end{enumerate}}}

    \vspace{0.3cm}
    
	\caption{Learning Margin Halfspaces with Massart Noise} \label{alg:massart}
\end{Ualgorithm}

\Cref{alg:massart} employs online SGD applied to a sequence of convex loss functions.
We show that, after a certain number of iterations, the algorithm will find a weight vector achieving 0-1 error at most $\eta+\eps$. Since the desired vector may not be the last iterate, 
in the end, our algorithm returns the halfspace that achieves the smallest empirical 0-1 error.

We establish the following result, which implies \Cref{thm:main-inf}.

\begin{theorem}[Main Result]\label{thm:main-detailed} 
Let $D$ be a distribution on $\mathbb{S}^{d-1} \times \{\pm 1\}$ satisfying
the $\eta$-Massart noise condition with respect to the $\gamma$-margin 
halfspace $f(\x) = \sign(\wstar\cdot\x)$. 
Given $N=\Theta(\log(1/(\gamma\delta))/\eps(1-2\eta))$ and $T=\Theta(\log(1/\delta)/(\eps^2\gamma^2))$,
\Cref{alg:massart} returns a vector $\hat{\vec w}$ such that 
$\mathrm{err}_{\D}(\hat{\vec w})\leq \eta+\eps$ with probability at least $1-\delta$. 
The algorithm draws $n=O(N+T)$ samples from $D$ 
and runs in $O(d N T)$ time.
\end{theorem}

The rest of this section is devoted to the proof of \Cref{thm:main-detailed}.

Our algorithm sets $\vec v=\vec w^t$ in each round, therefore for the rest of the section 
we proceed by setting $\vec v=\vec w$ as arguments of $\vec g$ and $\vec G$.

We decompose the stochastic gradient $\vec g(\w,\vec w,\x,y)$ into two parts: $\vec g(\w,\w,\x,y)=\vec g^1(\w,\x)+\vec g^2(\w,\x,y)$, where 
\[ \vec g^1(\w,\x)=\bigg((1-2\eta)\sign(\w\cdot\x)-\E_{y\sim D_y(\x)}[y]\bigg) W(\w\cdot\x,\gamma/2)\x\] and 
\[\vec g^2(\w,\x,y)=\bigg(\E_{y\sim D_y(\x)}[y]-y\bigg) W(\w\cdot\x,\gamma/2)\x\;.\]
We also use $\vec G_{\widehat{D}_N}^1(\w)$ and $\vec G_{\widehat{D}_N}^2(\w)$ 
for the same decomposition after taking the empirical expectation, 
i.e., $\vec G_{\widehat{D}_N}^1(\w)=\E_{\x\sim (\widehat{D}_\x)_N}[\vec g^1(\w,\x)]$ and 
$\vec G_{\widehat{D}_N}^2(\w)=\E_{(\x,y)\sim \widehat{D}_N}[\vec g^2(\w,\x,y)]$.

This serves to decompose the gradient into two parts: 
one containing the population expectation over the random variable $y$, 
and the other containing the error between the empirical estimation of $y$ 
and the population version of $y$. 
The vector $\vec G_{\widehat D_N}^1(\vec w)$ 
contains the direction that will decrease the distance 
between $\vec w$ and $\wstar$, while $\vec G_{\widehat D_N}^2(\vec w)$ 
contains the estimation error. 
To see this, observe that if we take the population expectation of  
$\vec g^2(\vec w,\x,y)$, we will have: 
\[\Ey[\vec  g^2(\vec w,\x,y)]=
\Exx \bigg[ \bigg( (1-2\eta(\x))\sign(\wstar\cdot\x)-\E_{y\sim D_y(\x)}[y] \bigg) 
 W(\w\cdot\x,\gamma/2)\x \bigg]= 0 \;,\]
where we used that $\E_{y\sim D_y(\x)}[y] =(1-2\eta(\x))\sign(\wstar\cdot\x)$.

We start by bounding the contribution of $\vec G_{\widehat{D}_N}^1(\w)$ 
in the direction $\w-\wstar$. We show that if instead of the corrupted label $y$ at the point $\x$, we had access to 
$\E_{y\sim D_y(\x)}[y]=(1-2\eta(\x))\sign(\wstar\cdot\x)$, 
then the gradient has a large component in the direction of $\w-\wstar$. 
This effectively implies that $\vec G_{\widehat{D}_N}^1(\w)$ 
can be used as a separation oracle, 
separating all the halfspaces with 0-1 error more than $\eta+\eps$ 
from the ones with smaller error.

\begin{lemma}[Structural Lemma]\label{lem:linear-time}
 Let $N\in \Z_+$ and let $D$ be a distribution on $\mathbb{S}^{d-1} \times \{\pm 1\}$ satisfying  
 the $\eta$-Massart condition with respect to the optimal classifier $f(\x)=\sign(\wstar\cdot\x)$. Let $\vec w\in \R^d$ be such that 
 $\|\vec w\|_2\leq 1$ and let $\{\x^{(i)}\}_{i=1}^N$ be a multiset of $N$ i.i.d.\ samples from $\D_\x$. Then, it holds
    $
  \vec G_{\widehat{D}_N}^1(\w)\cdot (\w-\wstar)\geq 2(\mathrm{err}_{\widehat{D}_N}(\w)-\eta)\;,
    $
  where $\widehat{D}_N$ is the corresponding empirical distribution. 
\end{lemma}

\begin{proof}
    We partition $\R^d$ into two subsets $R_1,R_2$ as follows: 
    $R_1$ contains the points that lie sufficiently far away 
    from the separating hyperplane $\vec w \cdot \x = 0$, 
    i.e., $R_1\eqdef\{\x\in\R^d:|\vec w\cdot\x|\geq \gamma/2\}$. 
    $R_2$ contains the remaining points, i.e., 
    $R_2\eqdef\{\x\in\R^d:|\vec w\cdot\x|< \gamma/2\}$. 
    
    We first show that for any $\x\in R_1$, the vector $\vec g^1(\vec w,\x)$ 
    has a large component parallel to the direction $\vec w-\wstar$. The proof of the claim below can be found in \Cref{ssec:ommited}.
    \begin{claim}\label{clm:first-clm}
    For any $\x^{(i)} \in R_1$, we have that
    $
    \vec g^1(\vec w,\x^{(i)})\cdot (\w-\wstar)\geq 2(\mathrm{err}(\vec w,\x^{(i)})-\eta)\;.
    $
    \end{claim}

It remains to show that the same holds for all the points in $R_2$. The proof of the claim below can be found in \Cref{ssec:ommited}.
    \begin{claim}\label{clm:second-clm}
    For any $\x^{(i)}\in R_2$, we have that
    $
    \vec g^1(\vec w,\x^{(i)})\cdot (\w-\wstar)\geq 2(\mathrm{err}(\vec w,\x^{(i)})-\eta)\;.
    $
    \end{claim}

   Applying \Cref{clm:first-clm} and \Cref{clm:second-clm} for each sample in the set $\{\x^{(i)}\}_{i=1}^N$, we get that 
    \begin{align*}
         \frac 1N\sum_{i=1}^N \vec g^1(\vec w,\x^{(i)})\cdot (\w-\wstar) \geq   \frac 2N\sum_{i=1}^N (\mathrm{err}(\vec w,\x^{(i)})-\eta)\;.
    \end{align*}
This completes the proof of \Cref{lem:linear-time}.
\end{proof}

By \Cref{lem:linear-time}, the gradient points towards the direction 
$\vec w^t-\wstar$, in the $t$-th iteration. 
This means that, in fact, the gradient is a subgradient of the potential 
loss $\Phi(\vec w)=\|\vec w-\wstar\|_2^2$. This allows us to show 
convergence, even though it is generally not possible in a sequence of 
loss functions in the stochastic setting.
We are now ready to prove our main result. 

\begin{proof}[Proof of \Cref{thm:main-detailed}]
Let $T$ be the maximum number of iterations of \Cref{alg:massart}. 
Denote by $\mathcal{Z}^{t} := \{ (\vec x^{(t)},y^{(t)})\}$ 
the i.i.d.\ sample drawn from $D$ in the $t$-th iteration, $t\in[T]$.
Furthermore, let $\mathcal F_1,\ldots, \mathcal F_T$ be the filtration 
with respect to the $\sigma$-algebra generated by $\mathcal{Z}^{1},\ldots, \mathcal{Z}^T$.
We denote by $H_t$ the event that $\mathrm{err}_{D}(\w^t)\geq \eta+\eps$.

Recall that \Cref{alg:massart} uses the following update rule (see Step \eqref{alg:update}): 
\[\vec w^{t+1}=\proj_{\{\vec w\in \R^d:\|\vec w\|_2\leq 1\}}(\vec w^{t}-\lambda_t  \vec g(\vec w^t,\vec w^t,\x^{(t)},y^{(t)}) )\;,\] 
with $\lambda_t=c\gamma^2\eps\;,$
for some sufficiently small absolute constant $c>0$. 

We begin by bounding from above the distance between $\vec w^{t+1}$ and $\wstar$ from the previous distance between $\vec w^t$ and $\wstar$.
	We have that
	\begin{align}
		\|\vec w^{t+1}-\wstar\|_2^2&=\|\proj_{\{\vec w\in \R^d:\|\vec w\|_2\leq 1\}}(\vec w^{t}-\lambda_t \vec g(\vec w^t,\vec w^t,\x^{(t)},y^{(t)}) -\wstar\|_2^2 \nonumber
		\\&\leq \|\vec w^{t}-\lambda_t \vec g(\vec w^t,\vec w^t,\x^{(t)},y^{(t)}) -\wstar\|_2^2 \nonumber
		\\&=\|\vec w^{t}-\wstar\|_2^2 -2\lambda_t \vec g(\vec w^t,\vec w^t,\x^{(t)},y^{(t)}) \cdot(\vec w^t-\wstar) +\lambda_t^2\|\vec g(\vec w^t,\vec w^t,\x^{(t)},y^{(t)}) \|_2^2
\;,\label[ineq]{ineq:5}
	\end{align}
 where in the first inequality we used the projection inequality, i.e., $\|\proj_B(\vec v)-\proj_B(\vec u)\|_2\leq \|\vec v-\vec u\|_2$ for any set $B$.  
 We will decouple the mean of the random variable $\vec g(\vec w^t,\vec w^t,\x,y)$ and make it zero-mean.
 
 To simplify the notation, we denote by $ \xi_t := \bigg(\vec g(\vec w^t,\vec w^t,\x^{(t)},y^{(t)}) -\vec G^1_{\D}(\vec w^t)\bigg)\cdot(\vec w^t-\wstar)$ and note that $ \xi_t$ is a zero-mean random variable over the sample $(\x^{(t)},y^{(t)})$. 
Adding and subtracting $\vec G^1_{\D}(\vec w^t)$ 
 onto \Cref{ineq:5} a we get that
	\begin{align}
		\|\vec w^{t+1}-\wstar\|_2^2\leq \|\vec w^{t}-\wstar\|_2^2 \underbrace{-2
  \lambda_t \vec G_{ D}^1(\vec w^t)\cdot(\vec w^t-\wstar) +\lambda_t^2\|\vec g(\vec w^t,\vec w^t,\x^{(t)},y^{(t)}) \|_2^2}_{I}
  \underbrace{-2
  \lambda_t  \xi_t}_{\widehat{V}_t}\;.\label[ineq]{eq:loop}
	\end{align} 
We now outline the main steps of our analysis. 
Instead of accurately estimating the gradients in each round, 
we denote by $\widehat V_t$ the estimation error from which we bound above their sum. 
We first add and subtract the population gradient to obtain the $I$ term,
which is the decreasing direction. In this way, we decouple the expected decrease 
and the error of the approximation (see \Cref{clm:true-decrease}). 
After that, we bound the contribution of the estimation error in \Cref{clm:martingale}. 
Observe that $\widehat{V}_t$ is a random variable that corresponds 
to the estimation error of the gradient. We will argue that with high probability the contribution of $\sum_{t=1}^T\widehat{V}_t$ is bounded; 
therefore, our algorithm will converge to an accurate solution.

\Cref{lem:linear-time} shows that $\vec G^1_{\widehat D_N^t}(\vec w^t)$ 
(and therefore the same holds for $\vec G^1_{D}(\vec w^t)$) 
contains substantial contribution 
towards to the direction $\vec w^t-\wstar$, 
depending of the current error. 
We show that our choice of step size guarantees a decreasing direction.
To this end, we prove the following:
 \begin{claim}\label{clm:true-decrease}
  Assume that the event $H_t$ happens, i.e., $\mathrm{err}_D(\w^t)\geq\eta+ \eps$. 
  If $\lambda_t\leq \gamma^2\eps/8$, then 
       $I\leq -\lambda_t(\mathrm{err}_D(\w^t)-\eta)$.
 \end{claim}

 \begin{proof}[Proof of \Cref{clm:true-decrease}]
Recall that $I= -2\lambda_t \vec G_{ D}^1(\vec w^t)\cdot (\vec w^t-\wstar) +\lambda_t^2\|\vec g(\vec w^t,\vec w^t,\x^{(t)},y^{(t)}) \|_2^2$.
 By \Cref{lem:linear-time}, we get that 
 $ \vec G_{\widehat D_N}^1(\w^t)\cdot (\w^t-\wstar)\geq 2(\err_{\widehat D_N}(\w^t)-\eta)$; hence, by taking expectations over the samples, we also have 
 $\vec G_{\D}^1(\w^t)\cdot (\w^t-\wstar)\geq 2(\err_D(\w^t)-\eta)$. 
 Furthermore, we have that $\|\vec g(\vec w^t,\vec w^t,\x^{(t)},y^{(t)}) \|_2^2\leq 8/\gamma^2$.
 Hence, 
\[   I\leq  -2\lambda_t (\mathrm{err}_{D}(\w^t)-\eta)  +8(\lambda_t^2/\gamma^2)\;.\]
The claim follows by noting that if $\lambda_t\leq \gamma^2\eps/8$, then $-\lambda_t (\mathrm{err}_{D}(\w^t)-\eta)  +8(\lambda_t^2/\gamma^2)\leq 0$. Therefore, we obtain 
 \[
 I\leq  -\lambda_t (\mathrm{err}_{D}(\w^t)-\eta)\;.
 \]
 This completes the proof of \Cref{clm:true-decrease}.
 \end{proof}
Therefore, our choice of parameters guarantees that $\lambda_t\leq \gamma^2\eps/8$. 
Using \Cref{clm:true-decrease} onto \Cref{eq:loop},
we have that 
 	\begin{align}
		\|\vec w^{t+1}-\wstar\|_2^2&\leq \|\vec w^{t}-\wstar\|_2^2-\lambda_t(\mathrm{err}_D(\w^t)-\eta) +\widehat{V}_t\;.\label[ineq]{eq:loop-2}
	\end{align} 
 Using \Cref{clm:true-decrease} and \Cref{eq:loop-2}, we have that
 	\begin{align}
		\|\vec w^{T+1}-\wstar\|_2^2&\leq  \|\vec w^{T}-\wstar\|_2^2 -\lambda_T(\mathrm{err}_D(\w^T)-\eta)+ \widehat {V}_T
  \nonumber\\&\leq \|\vec w^{0}-\wstar\|_2^2-\sum_{t=0}^T\lambda_t(\mathrm{err}_D(\w^t)-\eta)+\sum_{t=0}^T\widehat {V}_t\;.\label[ineq]{eq:loop3}
	\end{align}
 To complete the proof of \Cref{thm:main-detailed}, we need to bound the estimation error that corresponds to the random variable $\widehat {V}_t$. We show that $\widehat {V}_t$ does not increase the error by a lot. 
 Recall that 
 $\widehat V_t= - \new{2}
  \lambda_t  \xi_t\;.$ 
 
Before proceeding,  
we provide some basic background on subgaussian random variables.

\begin{definition}[Subgaussian Random Variable]\label{def:subgaussian}
For $\sigma>0$, a zero-mean random variable $ X\in \R$ 
is called $\sigma$-subgaussian, if for any $\lambda\in \R$  
it holds 
$
\log(\E[\exp(\lambda X)])\leq \lambda^2\sigma^2\;.$
\end{definition}

Note that any zero-mean bounded random variable is subgaussian. 
Specifically, we have the following:
\begin{fact}[Hoeffding’s lemma, see, e.g.,~\cite{Ver18}]\label{fct:hoef}
Let $ X\in \R$ be a zero-mean random variable such that $|X|\leq \sigma$ for some $\sigma>0$. Then $ X$ is $C\sigma$-subgaussian, where $C>0$ is a universal constant.
\end{fact}

Equipped with the above context, we show the following: 
 \begin{lemma}\label{clm:martingale}
   With probability at least $1-\delta$ over the random samples, 
  it holds that $ \sum_{t=0}^T\widehat {V}_t\leq C\gamma^2\eps^2T+\log(1/\delta)$, 
  where $C>0$ is an absolute constant.
 \end{lemma} 
 \begin{proof}

We first show that $\xi_t$ is a subgaussian random variable.
 \begin{claim}\label{clm:subgaussian}
 The random vector $\xi_t$ is $(16/\gamma)$-subgaussian.
 \end{claim}
 \begin{proof}[Proof of \Cref{clm:subgaussian}]
Note that $\xi_t= (\vec g(\vec w^t,\vec w^t,\x^{(t)},y^{(t)})-\Ey[\vec g(\vec w^t,\vec w^t,\x,y)])\cdot(\vec w^t-\wstar)$ and that by construction $\|\vec g(\vec w^t,\vec w^t,\x,y)\|_2\leq 4/\gamma$. Therefore, it holds that 
$$|\vec g(\vec w^t,\vec w^t,\x^{(t)},y^{(t)})\cdot(\vec w^t-\wstar)|\leq 8/\gamma \;,$$ 
where we used that $\|\vec w^t-\wstar\|_2\leq 2$ as both of these vectors lie in the unit ball. Hence,  by \Cref{fct:hoef}, we have that $\xi_t$ is $(16/\gamma)$-subgaussian.
   \end{proof}
    Using \Cref{clm:subgaussian} and \Cref{def:subgaussian} with parameter $\lambda= -2\lambda_t$ and $X=\xi_t$, we have that
          \[\log\E[\exp(\widehat {V}_t)]=
     \log\E[\exp(-2\lambda_t\xi_t)]\leq C(\lambda_t^2/\gamma^2)\;,
     \]
where $C>0$ is a universal constant.
 To bound the contribution of $\sum_{t=0}^T\widehat {V}_t$, we use Markov's inequality with respect to the filtration $\mathcal{F}_1,\ldots,\mathcal{F}_T$. We have that for any $Z\in \R$, it holds that
 \begin{align*}
     \pr_{\mathcal{Z}^1,\ldots,\mathcal{Z}^T\sim D}\left[\sum_{t=0}^T\widehat {V}_t\geq Z \right]
     &= \pr_{\mathcal{Z}^1,\ldots,\mathcal{Z}^T\sim D}\left[\exp\left(\sum_{t=0}^T\widehat {V}_t\right)\geq \exp(Z) \right]\\
     &\leq \E_{\mathcal{Z}^1,\ldots,\mathcal{Z}^T\sim D}\left[\exp\left(\sum_{t=0}^T\widehat {V}_t\right)\right] 
     \exp(-Z)\\
     &=\prod_{t=1}^T\E_{\mathcal{Z}^t\sim D}\left[\exp\widehat {V}_t\mid \mathcal F_t \right] \exp(-Z)\leq \exp\left(C\sum_{t=0}^T \frac{\lambda_t^2}{\gamma^2}-Z\right)\;,
 \end{align*}
 where in the second inequality we use the independence of $\widehat {V}_t$ with $\{\widehat {V}_k\}_{k=1}^{t-1}$ with respect to the filtration $\mathcal F_t$. 
 Recalling that  $\lambda_t=c\gamma^2\eps$, where $c>0$ is a sufficiently small universal constant, we have that 
 \[
     \pr_{{\mathcal{Z}^1,\ldots,\mathcal{Z}^T\sim D}}\left[\sum_{t=0}^T\widehat {V}_t\geq Z \right]
     \leq \exp\left(Cc^2\gamma^2\eps^2T-Z\right)
     \leq \exp\left(Cc^2\gamma^2\eps^2 T-Z\right)\;.
 \]
 Setting $Z=Cc^2\gamma^2\eps^2 T+\log(1/\delta)$ and taking $c$ to be a sufficiently small absolute constant (as is done in our algorithm),  we get that $\pr_{\mathcal{Z}^1,\ldots,\mathcal{Z}^T\sim D}\left[\sum_{t=0}^T\widehat {V}_t\geq Z \right]\leq \delta$. This completes the proof of \Cref{clm:martingale}.
 \end{proof}
Assume that until the round $T$ the event $H_T$ holds, i.e., for all $i\in [T]$ we have that  $\mathrm{err}_{\D}(\vec w^i)\geq \eta+\eps$. Using \Cref{clm:martingale} onto \Cref{eq:loop3}, 
with probability at least $1-\delta$, we have that:
  	\begin{align*}
		\|\vec w^{T+1}-\wstar\|_2^2&\leq \|\vec w^{0}-\wstar\|_2^2-\sum_{t=0}^T\lambda_t(\mathrm{err}_D(\w^t)-\eta)+\sum_{t=0}^T\widehat {V}_t
  \\&\leq \|\vec w^{0}-\wstar\|_2^2-cT\eps^2\gamma^2+\log(1/\delta)\;. 
	\end{align*}
Running the algorithm for 
 $T=\Theta(\log(1/\delta)/(\eps^2\gamma^2))$ iterations guarantees that 
 with probability at least $1-\delta$, we will have that 
 $\|\vec w^{T+1}-\wstar\|_2^2\leq 0$, which means $\vec w^{T+1}=\wstar$. 
 In that case, i.e., in the case where all the events $H_i$ for $i\in[T]$ hold, $\vec w^{T+1}$ achieves the same error as the optimal halfspace, thus it has 0-1 error of at most $\eta+\eps$. 
 Therefore, at least one vector $\vec w^{t'}$ with $t'\in[T+1]$ achieves 0-1 error of at most $\eta+\eps$. 
 The algorithm, in Step \eqref{alg:tournament}, returns a vector $\widehat{\w}$ that has 0-1 error at most 
 $\err_D(\widehat \w)\leq \min_{t\in[T+1]}\err_D(\w^t)+\eps\leq 
 \eta+2\eps$. The algorithm requires $N = O(\log(T/\delta)/(\eps(1-2\eta)))$ 
 samples for Step \eqref{alg:tournament}, due to \cite{Massart2006}. The 
 algorithm draws a sample in each round and runs for at most $T$ rounds. 
 Therefore, \Cref{alg:massart} draws $n= N+T = \widetilde 
 O(\log(1/\delta)/(\eps^2\gamma^2))$ samples. 
The algorithm needs to test each of the $T$ hypotheses 
 with $N$ samples to find the closest one. Therefore, the  
 total runtime is $O(d T N)$ (as in the other subroutines  
 the algorithm uses the samples only to estimate the gradients $\vec g$, 
 which requires $O(1)$ additions of $d$-dimenional vectors). 
 This completes the proof of 
 \Cref{thm:main-detailed}.
\end{proof}

\section{Conclusions and Open Problems} \label{sec:conc}
In this paper, we give the first sample near-optimal and computationally efficient 
algorithm for learning margin halfspaces in the presence of Massart noise. 
Specifically, the sample complexity of our algorithm nearly matches the 
computational sample complexity of the problem and its computational complexity is 
polynomial in the sample size. 
An interesting direction for future work is to develop a sample near-optimal 
and computationally efficient learner for 
general halfspaces (i.e., without the margin assumption). While our approach can likely 
be leveraged to obtain an efficient algorithm with sample complexity 
$\poly(d)/\eps^2$, the sample dependence on the dimension $d$ would be suboptimal. 
Obtaining the right dependence on the dimension 
seems to require novel ideas, as prior works rely on fairly 
sophisticated methods~\cite{DV:04, DKT21,DiakonikolasTK23} 
to effectively reduce to the large margin case.

\bibliographystyle{alpha}
\bibliography{clean2}
\appendix

\newpage

\section*{Appendix}

\paragraph{Organization} 
The structure of this appendix is as follows: 
In \Cref{app:rel-prior}, we provide additional summary and comparison 
with related and prior work. 
In \Cref{ssec:online}, we provide a polynomial time cutting-planes based algorithm with 
sample complexity $\widetilde O(1/(\eps^2\gamma^4))$. Finally, 
in \Cref{ssec:ommited}, we provide the proofs omitted from \Cref{sec:main}.

\section{Related and Prior Work}  \label{app:rel-prior}

\subsection{Additional Related Work} \label{app:related}

The computational problem of learning halfspaces with Massart noise
has been extensively studied, both in the distribution-specific
and the distribution-free settings. 

In the distribution-specific setting, 
the first efficient algorithm for homogeneous Massart halfspaces 
was given in \cite{AwasthiBHU15}. Subsequent work generalized this result in various directions~\cite{AwasthiBHZ16, ZhangLC17, 
YanZ17, DKTZ20, DKTZ20b, DKKTZ20, DKKTZ21b, DiakonikolasKKT22}. 

The first algorithmic progress in the distribution-free setting was made by~\cite{DGT19},
answering a longstanding open problem~\cite{Sloan88,Sloan92, Blum03}. Subsequent work
gave an algorithm with improved sample complexity~\cite{CKMY20} and provided strong
evidence that an error of $\eta+\eps$ is the best to hope 
for in polynomial time~\cite{DK20-SQ-Massart,NasserT22,DKPR22} 
(in both the Statistical Query model 
and under plausible cryptographic assumptions). In a related direction, 
\cite{DIKLST21} gave the first efficient boosting algorithm in the presence of Massart noise, 
which can boost a weak learner to one with error $\eta+\eps$. Finally, we note that 
natural generalizations of the Massart model to learning real-valued functions 
(in an essentially distribution-free setting) have also 
been studied~\cite{chen2021online,DJT21, DiakonikolasKRS22}.

Very recent work~\cite{DDKWZ23} gave SQ (and low-degree polynomial testing)
lower bounds for learning $\gamma$-margin halfspaces with RCN~\cite{AL88}, 
which is a special case of Massart noise. 
Specifically,~\cite{DDKWZ23} showed that any efficient SQ algorithm for the 
problem requires sample complexity 
$\Omega(1/(\gamma^{1/2}\eps^2))$. Subsequently,~\cite{DDKWZ23b} showed
a related SQ lower bound under the Gaussian distribution, which can be adapted
to obtain a lower bound of $\Omega(1/(\gamma \eps^2))$ for the margin setting.

\subsection{Comparison with \cite{DKTZonline}}\label{ssec:related}
The work \cite{DKTZonline} uses a similar sequence of loss functions 
for the problem of ``online learning''  Massart margin halfspaces. 
Intuitively, their goal is to minimize regret in an adversarial online setting. 
In their online setting, the adversary in each round commits 
to covariates $\x^1,\x^2\in\R^d$ and distribution $D^t$ over $\R_+\times\R_+$. 
Then the algorithm observes the covariates, chooses an action $a\in\{1,2\}$, 
and observes a reward $r_a\in \R_+$. It is only guaranteed that there exists 
a unit vector $\wstar$ so that $\E_{(r_1,r_2)\sim D^t}[\sign(\wstar\cdot\x^1-\wstar\cdot\x^2)(r_a-r_b)]\geq \Delta$ for some $\Delta>0$.

Despite this superficial similarity, the work of \cite{DKTZonline} 
has no new implications on the sample complexity of PAC learning Massart 
halfspaces with a margin. Specifically, 
they achieve a regret bound of $O(T^{3/4}/\gamma)$. 
If one translates this bound to a sample complexity upper bound for PAC learning, 
one would obtain a bound of $\Omega(1/(\eps^4\gamma^8))$ ---  
which is quantitatively worse than prior work of~\cite{DGT19, CKMY20}. 

At a technical level, our work leverages this sequence of loss functions as subgradients 
of the potential function $\Phi(\vec w)=\|\w-\wstar\|_2^2$. 
Via a novel analysis, we show that these subgradients $\Omega(\eps)$-correlate 
with the direction of $\w-\wstar$. This in turn means that we can expect 
a decrease of order $\Omega(\lambda \eps)$ in each iteration, 
where $\lambda$ is the corresponding step-size, 
as long as we get 0-1 error more than $\eta+\eps$. 
This structural understanding suffices for obtaining an algorithm, 
based on a separation oracle, 
that achieves a sample complexity of $\widetilde O(1/(\gamma^4\eps^2))$. 
In order to obtain an algorithm with near-optimal sample complexity (and runtime), 
we required additional new ideas as elaborated in the body of the paper.

\section{Learning Margin Massart Halfspaces via Cutting Planes}
\label{ssec:online}

In this section, we show how to use the cutting-planes method 
along with \Cref{lem:linear-time} to efficiently learning margin 
Massart Halfspaces using $\widetilde O(1/(\gamma^4\eps^2))$ samples.

Specifically, we establish the following result: 

\begin{theorem}[Learning Margin Massart Halfspaces with Cutting Planes]\label{thm:onlinemain} 
Let $D$ be a distribution on $\mathbb{S}^{d-1} \times \{\pm 1\}$ which satisfies the 
$\eta$-Massart noise condition with respect to the $\gamma$-margin 
halfspace $f(\x) = \sign(\wstar\cdot\x)$. 
Given $N=\Theta(\log(1/(\gamma\delta)/(\gamma^4\eps^2))$ i.i.d.\ samples from $\D$, 
there is a $\poly(d,N)$ time algorithm that returns a vector $\hat{\vec w}$ such that 
$\mathrm{err}_{\D}(\hat{\vec w})\leq \eta+\eps$ with probability at least $1-\delta$. 
\end{theorem}

\begin{remark}
{\em   We can always assume that $d=\widetilde O(1/\gamma^2)$. 
This holds since we can efficiently preprocess the data, 
using the Johnson-Lindenstrauss transform~\cite{JohnsonLindenstrauss:84}. 
Similar dimension-reduction steps have been use in prior work, e.g.,~\cite{CKMY20, DDKWZ23}.}
\end{remark}

Given the above remark, it suffices to establish the following:

\begin{theorem}\label{thm:online} 
Let $D$ be a distribution on $\mathbb{S}^{d-1} \times \{\pm 1\}$ which satisfies the 
$\eta$-Massart noise condition with respect to the $\gamma$-margin 
halfspace $f(\x) = \sign(\wstar\cdot\x)$. 
Given $N=\Theta(d\log(1/(\gamma\delta)/(\gamma^2\eps^2))$ i.i.d.\ samples from $\D$, 
there is a $\poly(d,N)$ time algorithm that returns a vector $\hat{\vec w}$ such that 
$\mathrm{err}_{\D}(\hat{\vec w})\leq \eta+\eps$ with probability at least $1-\delta$. 
\end{theorem}

The idea of using the cutting plane method is slightly adapted from \cite{CKMY20}.
Given access to a separation oracle for a convex set $\mathcal K$, 
we can find a point inside the set $\mathcal K$ by querying the separation oracle $O(d\log d)$ times. The difference with \cite{CKMY20} is that we are using 
a more sophisticated (and sample efficient) separation oracle.  
This allows us to use $O(1/\eps^2)$ samples, instead of $O(1/\eps^3)$ samples, and leads
to the optimal sample complexity as a function of $\eps$ (but not $\gamma$).

\begin{fact}\label{fct:cutting-planes}
    Suppose that $\mathcal K$ is an (unknown) convex body in $\R^d$ which contains a Euclidean ball of radius $r> 0$ and contained in a Euclidean ball centered at
the origin of radius $R> 0$. There exists an algorithm which, given access to a separation oracle
for $\mathcal K$, finds a point $\x^\ast \in \mathcal K$, runs in time $\poly(\log(R/r), d)$, and makes $O(d \log(Rd/r))$ calls to the
separation oracle.
\end{fact}

We first show that if we get enough samples, we can 
efficiently approximate the gradients $\vec G(\vec w,\vec w)$. 
Formally, we have: 

\begin{proposition}[Separation Oracle]\label{corr:online}
   Let $\eps,\delta\in(0,1)$ and let $D$ be a distribution on $\mathbb{S}^{d-1} \times\{\pm 1\}$ satisfying the $\eta$-Massart noise condition with respect to the halfspace  $f(\x)=\sign(\wstar\cdot\x)$. Fix $\vec w\in \R^d$ with $\|\vec w\|_2\leq 1$. Let $N \gtrsim \log(1/(\gamma\delta))/(\eps^2\gamma^2))$ and $\widehat{D}_N$ be the corresponding empirical distribution. 
     Then, with probability at least $1-\delta$, it holds that
    \[
  \vec G_{\widehat{D}_N}(\w,\w)\cdot (\w-\wstar)\geq 2(\mathrm{err}_{D}(\w)-\eta)-\eps\;.
    \] 
\end{proposition}
\begin{proof}
    By construction, 
    $\vec G_{\widehat{D}_N}(\w,\w)=\vec G^1_{\widehat{D}_N}(\w)+\vec G^2_{\widehat{D}_N}(\w)$ and by~\Cref{lem:linear-time} 
    we have that
    $\vec G_{\widehat{D}_N}^1(\w)\cdot (\w-\wstar)\geq 2(\mathrm{err}_{\widehat{D}_N}(\w)-\eta)$. 
    By definition, we have $\E_{(\x^{(1)},y^{(1)}), \ldots,  (\x^{(N)},y^{(N)})\sim \D}[\vec  G^2_{\widehat{D}_N}(\w)]=0$, where the expectation is taken with respect to the sample set. 
    Note that the norm of $\vec g^1(\w,\x),\vec g^2(\w,\x,y)$, i.e., $\|\vec g^1(\w,\x)\|_2,\|\vec g^2(\w,\x,y)\|_2$, is bounded pointwise
    from above by $4/\gamma$ for all $\vec w\in \R^d$. 
    This can be seen as $\|\x\|_2\leq 1$, $W(\cdot,\gamma/2)\leq 2/\gamma$, 
    and $(1-2\eta),(1-2\eta(\x))\leq 1$. 
    
    We use the following concentration inequality to show that our sample size is enough to guarantee that the estimated gradient is close to its population version.
    
    \begin{fact}[\cite{SZ07}, Lemma 1]\label{fct:Bennet-vectors}
    Let $\vec Z_1,\ldots,\vec Z_n\in \R^d$ be random vectors such that for each $i\in[n]$ it holds $\|\vec Z_i\|_2\leq M < \infty$ almost surely and let $\sigma^2=\sum_{i=1}^n\E[\|\vec Z_i\|_2^2]$. Then, we have that for any $\eps > 0,$
    \[
    \pr\left[\left\|\frac{1}{n}\sum_{i=1}^n \left(\vec Z_i- \E[\vec Z_i]\right)\right\|_2\geq \eps\right]\leq 2\exp\left(-\frac{n\eps}{2M}\log\bigg(1+\frac{n M\eps}{\sigma^2}\bigg)\right)\;.
    \]
    \end{fact}
   Using \Cref{fct:Bennet-vectors}, along with the inequality $\log(1+z)\geq z/2$, for $z\in(0,1)$, we get that if  $N\geq \Theta(\frac{\log(1/\delta)}{(\eps\gamma)^2})$, with probability at least $1-\delta$, we have 
      \begin{equation}
             \left\| \vec G^1_{\widehat{D}_N}(\w)-\Ey[\vec g^1(\w,\x)] \right\|_2\leq \eps\;,\label[ineq]{ineq:1}
      \end{equation}   
   and
   \begin{equation}
       \left\|\vec G^2_{\widehat{D}_N}(\w)-\Ey[\vec g^2(\w,\x,y)] \right\|_2\leq \eps\;.\label[ineq]{ineq:2}   
   \end{equation}

To complete the proof, 
recall that by \Cref{lem:linear-time} it holds 
$\vec G^1_{\widehat{D}_N}(\w)\cdot (\w-\wstar)\geq 2(\mathrm{err}_{\widehat D_N}(\w)-\eta)-\eps$. 
Therefore, by taking the expectation over $D_\x$, we get that
\[
\vec G^1_{D}(\w)\cdot (\w-\wstar)\geq 2(\mathrm{err}_{D}(\w)-\eta)\;.
\]
The proof is completed by recalling that 
$\|\vec G^1_{\widehat{D}_N}(\w)-\Ey[\vec g^1(\w,\x)]\|_2\leq \eps$ from \Cref{ineq:1} and that $\Ey[\vec g^2(\w,\x,y)]=0$.
\end{proof}

Equipped with \Cref{corr:online}, we are ready to prove a weaker version of \Cref{thm:main-detailed} using separation oracles and the cutting plane algorithm. Formally, we show that

\begin{proof}[Proof of \Cref{thm:online}]
Our convex set $\mathcal K$ is a Euclidean ball 
of radius $\gamma/2$ centered at $\wstar$. To see that, 
note that for any $\vec v$ such that $\|\wstar-\vec v\|_2\leq \gamma/2$, we have that 
$|(\wstar-\vec v)\cdot\x|\leq \gamma/2$ for any $\x$ with $\|\x\|_2=1$. 
This implies that
$\gamma/2+\wstar\cdot \x\geq \vec v\cdot\x\geq \wstar\cdot \x -\gamma/2$. 
Moreover, by definition we have that $\wstar\cdot \x\geq \gamma$. 
Hence, if $\wstar\cdot\x\geq 0$, we have that 
$\vec v\cdot\x\geq \gamma/2$; 
and if  $\wstar\cdot\x\leq 0$, we have that $\vec v\cdot\x\leq- \gamma/2$. 
Therefore, this ball contains all the vectors $\vec w$ 
with margin $\gamma/2$ and separates the points in the same way as $\wstar$.

Therefore, as long as we are not in the set $\mathcal K$ 
or the 0-1 error is more than $\eta+\eps$, 
we can use \Cref{corr:online} to construct a new separation oracle. 
By \Cref{fct:cutting-planes}, the maximum number of calls 
to the separation oracle is $T=O(d\log(d/\gamma))$. 
By \Cref{corr:online}, in each round we need 
$n=O(\log(T/\delta))/(\eps^2\gamma^2)$ samples from $\D$ 
to construct a separation oracle. 
Therefore, the maximum number of samples is 
$O(n T)=O(d \log(T/\delta))/(\eps^2\gamma^2)$. 
This completes the proof.
\end{proof}

\section{Omitted Proofs from \Cref{sec:main}}\label{ssec:ommited}

\subsection{Proof of \Cref{clm:dgt-claim}}

\begin{claim}[Claim 2.1 \cite{DGT19}]\label{clm:dgt-claim}
    For any $\vec w,\vec x$, we have that 
    \[\ell_{\lambda}(\vec w,\x,y)=\big(\1\{y(\vec w\cdot\x)\leq 0\}-\lambda\big)|\vec w\cdot\x| \;.\]
\end{claim}
\begin{proof}
    Recall that 
    \[ \ell_{\lambda}(\vec w,\x,y)=\mathrm{LeakyReLU}_\lambda(-y(\vec w\cdot\x))=(1-\lambda)\1\{y(\vec w\cdot\x)\leq 0\}(-y\vec w\cdot\x)+\lambda\1\{y(\vec w\cdot\x)>0\}(-y\vec w\cdot\x) \;.\]
    Therefore, we have that
    \begin{align*}
          \ell_{\lambda}(\vec w,\x,y)&=(1-\lambda)\1\{y(\vec w\cdot\x)\leq 0\}|y\vec w\cdot\x|-\lambda\1\{y(\vec w\cdot\x)>0\}|y\vec w\cdot\x|
      \\&=\1\{y(\vec w\cdot\x)\leq 0\}|\vec w\cdot\x|-\lambda|\vec w\cdot\x|=\bigg(\1\{y(\vec w\cdot\x)\leq 0\}-\lambda\bigg)|\vec w\cdot\x|\;,
    \end{align*}
where we used that $y\in\{\pm 1\}$.
\end{proof}
\subsection{Proof of \Cref{clm:first-clm}}
We restate and prove the following claim:

   {\bf Claim 2.3.}\quad {\em For any $\x^{(i)} \in R_1$, we have that}
    $
    \vec g^1(\vec w,\x^{(i)})\cdot (\w-\wstar)\geq 2(\mathrm{err}(\vec w,\x^{(i)})-\eta)\;.
    $

    \begin{proof}[Proof of~\Cref{clm:first-clm}]
    For any $\x^{(i)} \in R_1$, we have that
                \begin{align}
        \vec g^1(\vec w,\x^{(i)})\cdot \w&= \bigg((1-2\eta)\sign(\w\cdot\x^{(i)})-(1-2\eta(\x^{(i)}))\sign(\wstar\cdot\x^{(i)})\bigg)\w\cdot\x^{(i)}W(\w\cdot\x^{(i)})\nonumber
        \\&=\bigg((1-2\eta)\sign(\w\cdot\x^{(i)})-(1-2\eta(\x^{(i)}))\sign(\wstar\cdot\x^{(i)})\bigg)\sign(\w\cdot\x^{(i)})\nonumber
        \\&=2(\mathrm{err}(\vec w,\x^{(i)})-\eta)\;, \label{eq42:1}
    \end{align}
    where we used that for any $\x^{(i)}\in R_1$, 
    $W(\w\cdot\x^{(i)})=1/|\w\cdot\x^{(i)}|$, 
    and hence $W(\vec w\cdot \x^{(i)},\gamma/2)\w\cdot\x^{(i)}=\sign(\w\cdot\x^{(i)})$;  
    and that $\mathrm{err}(\vec w,\x^{(i)})=\eta(\x^{(i)})$ if $\sign(\w\cdot\x^{(i)})=\sign(\wstar\cdot\x^{(i)})$ and $1-\eta(\x^{(i)})$ otherwise. 
    
    We now bound the contribution of $\wstar$.
    Since $\eta(\x)\leq \eta$, we have 
    $$(1-2\eta(\x))-(1-2\eta)\sign(\w\cdot\x)\sign(\wstar\cdot\x) \new{\geq} 0 \;.$$ 
    Therefore, we have that
    \begin{align*}
       \vec g^1(\vec w,\x^{(i)})\cdot \wstar&=  \bigg((1-2\eta)\sign(\w\cdot\x)-(1-2\eta(\x))\sign(\wstar\cdot\x)\bigg)\sign(\wstar\cdot\x)|\wstar\cdot\x|W(\w\cdot\x^{(i)})
       \\&= -\bigg((1-2\eta(\x))-(1-2\eta)\sign(\w\cdot\x)\sign(\wstar\cdot\x)\bigg)|\wstar\cdot\x|W(\w\cdot\x^{(i)})\leq 0\;,
    \end{align*}
which gives that $-\vec g^1(\vec w,\x^{(i)})\cdot \wstar\geq 0$. This completes the proof of \Cref{clm:first-clm}.
    \end{proof}
    \subsection{Proof of \Cref{clm:second-clm}}
  We restate and prove the following:
  
    {\bf Claim 2.4.}\quad {\em For any $\x^{(i)}\in R_2$, we have that}
    $
    \vec g^1(\vec w,\x^{(i)})\cdot (\w-\wstar)\geq 2(\mathrm{err}(\vec w,\x^{(i)})-\eta)\;.
    $
\begin{proof}[Proof of \Cref{clm:second-clm}]
        We have that
               \begin{align}
          \vec g^1(\vec w,\x^{(i)})\cdot (\w-\wstar) &= \bigg((1-2\eta)\sign(\w\cdot\x^{(i)})-(1-2\eta(\x^{(i)}))\sign(\wstar\cdot\x^{(i)})\bigg)\left(\frac{\w\cdot\x^{(i)}-\wstar\cdot\x^{(i)}}{\max(\gamma/2,|\w\cdot \x^{(i)}|)}\right)\nonumber
          \\&=  \bigg((1-2\eta)\sign(\w\cdot\x^{(i)})-(1-2\eta(\x^{(i)}))\sign(\wstar\cdot\x^{(i)})\bigg)\left(\frac{\w\cdot\x^{(i)}-\wstar\cdot\x^{(i)}}{\gamma/2}\right)\nonumber\;,
    \end{align}
    where we used that $\max(\gamma/2,|\w\cdot\x^{(i)}|)=\gamma/2$ for any $\x^{(i)}\in R_2$.
    Since $\sgn(\wstar\cdot \x)$ has $\gamma$-margin, we have that $|\wstar\cdot\x^{(i)}|\geq \gamma$. Since $\x^{(i)}\in R_2$, it holds $|\vec w\cdot\x^{(i)}| < \gamma/2$. Therefore,  
    $-\sign(\wstar\cdot\x^{(i)})(\w\cdot\x^{(i)}-\wstar\cdot\x^{(i)})
    =\left(|\wstar\cdot\x^{(i)}|-\sign(\wstar\cdot\x^{(i)})\w\cdot\x^{(i)} \right) 
    \geq \gamma/2$. This in turn implies that
               \begin{align}
          \vec g^1(\vec w,\x^{(i)})\cdot (\w-\wstar) &\geq(1-2\eta(\x^{(i)})-(1-2\eta)\sign(\w\cdot\x^{(i)})\sign(\wstar\cdot\x^{(i)}))\nonumber
          \\&=2(\mathrm{err}(\vec w,\x^{(i)})-\eta)\nonumber\;,
    \end{align}
    completing the proof of \Cref{clm:second-clm}.
\end{proof}

\end{document}